\documentclass[lettersize,journal]{IEEEtran}
\usepackage{amsmath,amsfonts}
\usepackage{algorithmic}
\usepackage{algorithm}
\usepackage{array}
\usepackage[caption=false,font=footnotesize,labelfont=rm,textfont=rm]{subfig}
\usepackage{textcomp}
\usepackage{stfloats}
\usepackage{changes}
\usepackage{url}
\usepackage{verbatim}
\usepackage{graphicx}
\usepackage{multirow}
\usepackage{makecell}
\usepackage{cite}

\newtheorem{proposition}{Proposition}
\newenvironment{proof}{{\noindent\it Proof} \quad}{\hfill $\square$\par}
\begin{document}

\title{Towards Accurate Post-training Quantization\\ for Reparameterized Models}

\author{Luoming Zhang, Yefei He, Wen Fei, Zhenyu Lou, Weijia Wu, YangWei Ying, and Hong Zhou
\thanks{This work was supported by National Key Research and Development Program of China (2022YFC3602601), and Key Research and Development Program of Zhejiang Province of China (2021C02037)}
\thanks{Luoming Zhang, Yefei He, Zhenyu Lou, Weijia Wu, Hong Zhou are with the Zhejiang Provincial Key Laboratory for Network Multimedia Technologies,
Zhejiang University, Hangzhou 310027, China, and also with the Zhejiang University Embedded System Engineering Research Center, Ministry of Education of China, Zhejiang University, Hangzhou 310027, China (e-mail: zluoming@zju.edu.en; billhe@zju.edu.cn; 11915044@zju.edu.cn; weijiawu@zju.edu.cn; zhouh@mail.bme.zju.edu.cn).}
\thanks{Wen Fei is with the Department of Electronic Engineering, Shanghai Jiao Tong University, Shanghai 200240, China (e-mail: fw.key@sjtu.edu.cn).}}


\markboth{Journal of \LaTeX\ Class Files,~Vol.~14, No.~8, August~2021}%
{Shell \MakeLowercase{\textit{et al.}}: A Sample Article Using IEEEtran.cls for IEEE Journals}


\maketitle

\begin{abstract}
  Model reparameterization is a widely accepted technique for improving inference speed without compromising performance. However, current Post-training Quantization (PTQ) methods often lead to significant accuracy degradation when applied to reparameterized models. This is primarily caused by channel-specific and sample-specific outliers, which appear only at specific samples and channels and impact on the selection of quantization parameters. To address this issue, we propose RepAPQ, a novel framework that preserves the accuracy of quantized reparameterization models. Different from previous frameworks using Mean Squared Error (MSE) as a measurement, we utilize Mean Absolute Error (MAE) to mitigate the influence of outliers on quantization parameters. Our framework comprises two main components: Quantization Protecting Reparameterization and Across-block Calibration. For effective calibration, Quantization Protecting Reparameterization combines multiple branches into a single convolution with an affine layer. During training, the affine layer accelerates convergence and amplifies the output of the convolution to better accommodate samples with outliers. Additionally, Across-block Calibration leverages the measurement of stage output as supervision to address the gradient problem introduced by MAE and enhance the interlayer correlation with quantization parameters. Comprehensive experiments demonstrate the effectiveness of RepAPQ across various models and tasks. Our framework outperforms previous methods by approximately 1\% for 8-bit PTQ and 2\% for 6-bit PTQ, showcasing its superior performance. The code is available at \url{https://github.com/ilur98/DLMC-QUANT}.
\end{abstract}


\section{Introduction}

\IEEEPARstart{I}{n} recent years, deep neural networks have made tremendous advancements in various domains. However, as the depth and breadth of these models increase, so does the number of parameters and computational cost. To address this issue, techniques such as quantization \cite{Esser:LEARNED,nagel:up}, pruning \cite{guo20223d,guo2020model}, low-rank tensor factorization \cite{zhou2017tensor,zhen2022towards}, and knowledge distillation \cite{hinton2015distilling} have been proposed to reduce the computational complexity while preserving model accuracy.

Quantization is a prominent technique for optimizing computational expenses and enhancing the inference efficiency of models. The primary focus of quantization involves reducing calculation precision, generally by converting the model of full-precision (FP32) to 8-bit (INT8) or lower. Numerous empirical investigations indicate that quantizing a model to 8 bits does not degrade its accuracy, and quantization-aware training (QAT)~\cite{bhalgat2020lsq+,Esser:LEARNED} can yield results comparable to full-precision models with only 4 bits. Post-training quantization (PTQ)~\cite{nagel:up,li:brecq} is a popular approach that uses a small set of unlabeled data to calibrate a pre-trained network without the need for additional training or fine-tuning, making it a viable solution for various scenarios.

\begin{figure}[!t]
  \centering
  \centerline{
      \includegraphics[width=1\columnwidth]{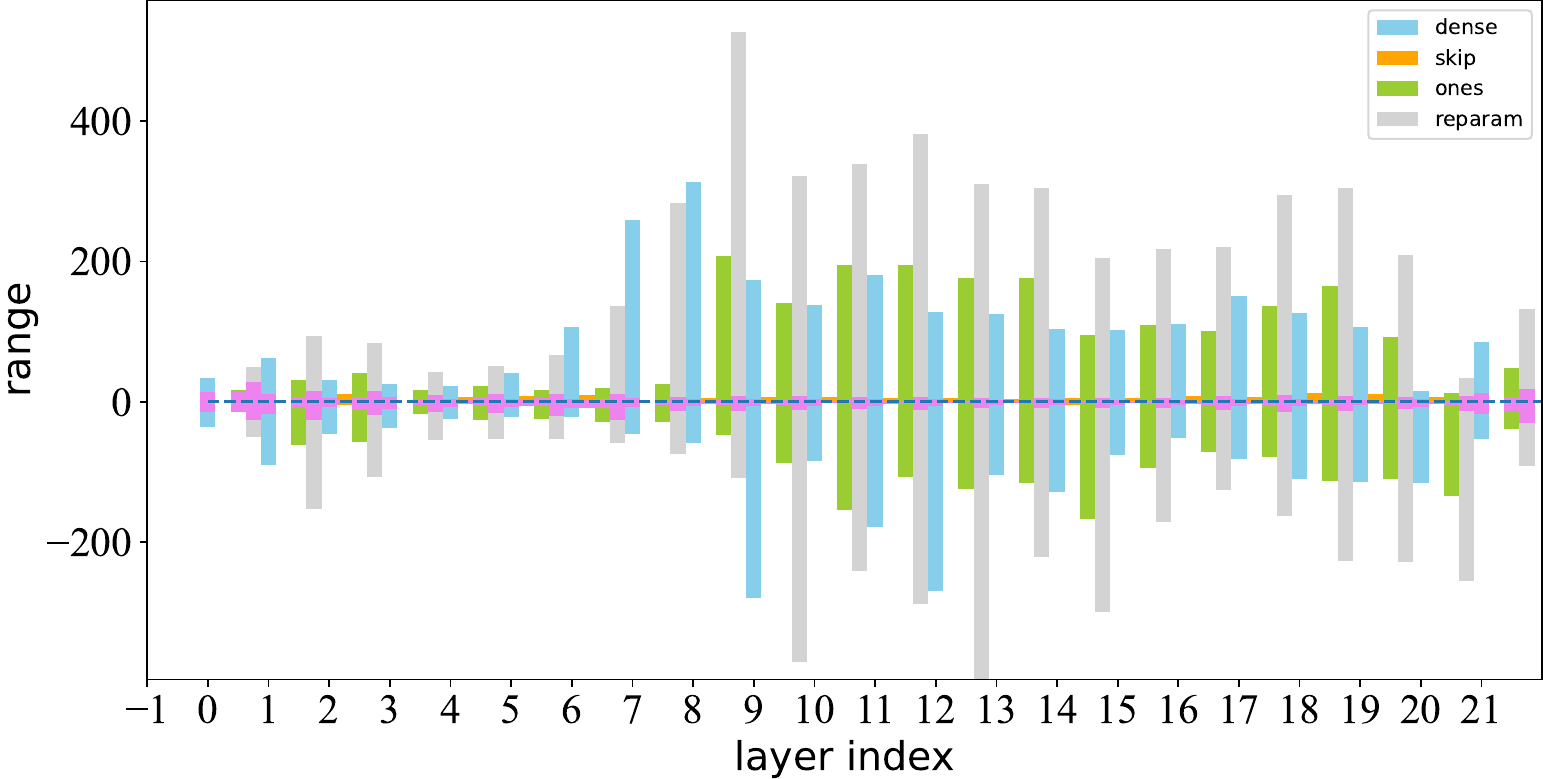}
  }
  \quad
  \centering
  \caption{Box plot for activation to quantize in RepVGGA0. The blue bar is the output of the 3x3 branch, the green bar is the output of the 1x1 branch, the orange bar is the output of the shortcut branch, the gray bar is the fused output, and the violent box is the box with a range of 99\%.}
  \label{fig_range_A0}
  \vskip -0.2in    
\end{figure}

Reparameterization is a novel approach in model design that increases the representational capacity of single convolution by incorporating branches with varying scales and computations. During the inference stage, all branches are linear and can be fused into a single convolution. Notable reparameterization models, such as RepVGG \cite{ding:repvgg} and MobileOne \cite{vasu:improved}, resembling plain models following the style of VGG \cite{simonyan2014very}. In comparison to residual networks under the same parameter budget, plain models are optimized for memory access and exhibit superior speed performance. Reparameterization has been harnessed in object detection and super-resolution tasks to minimize parameter counts while improving performance.

Naive quantization of reparameterized models can result in significant accuracy loss, even for 8-bit quantization, due to the BatchNorm layers~\cite{ioffe2015batch}. This leads to a high variance of reparameterized convolution outputs~\cite{ding:re,chu2022make}, as demonstrated by the activation range plot of RepVGGA0 in Fig.~\ref{fig_range_A0}. While the majority of values are small, there exist outliers that are significantly larger than the majority of values. For example, in the 9th layer, with a maximum value of 426.25, the 99\% threshold is 6.66. Additionally, the aggregate output is larger than the output of each branch. 
The design of Post-add BatchNorm layers is beneficial for accuracy but introduces larger outliers, making quantization more challenging. When using naive quantization methods for activation with these large outliers, the quantized values tend to cluster at the smallest quantized values. This is because the ratio of the maximum value to the median value can exceed 256 (the range of 8-bit quantization) in various layers, resulting in many values being mapped to a single quantized value. Consequently, it reduces the information entropy of the activation. Moreover, most existing quantization methods initialize parameters by minimizing the mean squared error (MSE) of each layer's output. However, the MSE between quantized values and full-precision values is larger for outliers compared to normal values, as the large errors are amplified by the quadratic function. This means that using MSE as a measurement fails to find the optimal quantization parameters for accurate quantized models. Additionally, distillation-based methods that utilize MSE for block output tend to introduce a large gradient for channels with outliers, leading to a significant increase in the gradient of weight scales. Consequently, existing quantization methods struggle to obtain accurate quantized models.

After conducting our analysis, we made several improvements to address the challenges posed by outliers in reparameterized models during quantization. We introduced a novel method called Quantization Protecting Reparameterization, which involves inserting an affine layer after convolution during the calibration process. The affine layer helps to amplify the features to accommodate samples with outliers better and accelerate the convergence of quantized models. Additionally, we developed the Across Block Calibration technique to overcome the gradient problem and establish interlayer correlations. By employing these methods, we successfully extended the capabilities of post-training quantization for reparameterized models, achieving accurate results even with 6-bit quantization.

This paper contributes in the following ways:
\begin{itemize}
  \item Review and identification: The paper provides a comprehensive review of quantized reparameterized models and identifies large activation outliers as a significant factor contributing to accuracy degradation. It highlights the limitations of existing techniques that aim to mitigate this issue by minimizing Mean Squared Error (MSE) during initialization and per-block calibration, emphasizing their failure to find optimal quantization parameters in the presence of outliers.
  \item Quantization Protecting Reparameterization (QPRep): To address the impact of large outliers during calibration, the paper proposes a novel method called QPRep. This method introduces a channel-wise affine layer following reparameterized convolutions. The affine layer adaptively amplifies channel features to better fit samples with outliers, thereby preserving accuracy during quantization.
  \item Across-block Calibration (ABC): The paper introduces ABC as an alternative calibration paradigm to per-block calibration using MSE. ABC utilizes Mean Absolute Error (MAE) between the current layer output and the stage output as a measure for distillation. By optimizing a stage that includes multiple blocks with ABC, more precise stage outputs can be achieved.
  \item Extensive evaluation: The proposed techniques are extensively evaluated on various model variants for image classification and object detection tasks. Through these evaluations, the paper demonstrates that the proposed methods outperform existing Post-training Quantization methods, resulting in significant performance improvements.
\end{itemize}

\section{Related Work} \label{relatedwork}
\subsection{Reparameterization}
Reparameterization has proven to be a powerful technique in efficient network design. For instance, DiracNet \cite{zagoruyko:diracnets} developed a Dirac parameterization strategy to train deep-plan networks without the need for explicit skip-connections. Ding \cite{ding:acnet} introduced an asymmetric convolution block (ACB) to replace standard square-kernel layers, employing 1D asymmetric convolution to enhance the square convolution kernels and fusing asymmetric convolutions during inference. RepVGG \cite{ding:repvgg} aimed to construct a VGG-like inference network using only $3 \times 3$ convolution and ReLU, and utilized a multi-branch topology containing a $1 \times 1$ convolution, a $3 \times 3$ convolution, and an identity mapping during training.
Another example is the work of Ding et al. \cite{ding:diverse}, which presented six reparameterization paradigms and devised a new method to enhance accuracy in response to the ACB structure. OREPA \cite{hu2022online} featured a scaling layer to reduce the substantial training overhead by compressing the complex training-time block into a single convolution. RepLKNet \cite{ding:scaling} and Repmlp \cite{ding:repmlp} developed large models using reparameterization modules.
Incorporating reparameterization into the efficient model design, \cite{vasu:improved} proposed the MobileOne model, which achieved an inference time under 1 ms. In object detection, YoLov6 \cite{li2022yolov6} and YoLov7 \cite{wang2022yolov7} introduced reparameterization, which resulted in better accuracy than other detectors with similar inference speeds.
Furthermore, super-resolution networks\cite{zhang:edge,wang:edge} utilized fixed window functions to enhance sensitivity to edge features.

RepOpt \cite{ding:re} and QARepVGG \cite{chu2022make} effort to reduce the weight variance, which introduces additional loss to full precision models and requires training a new model to replace the original one. However, these additional constraints will limit the accuracy of the model. Therefore, there is an urgent need to find a method to obtain accurate quantized reparameterized models using vanilla reparameterized models.

\subsection{Model Quantization}
Quantization-aware training (QAT) \cite{jacob:quantization,Esser:LEARNED,fan:training} involves re-training to prevent accuracy degradation during the quantization process, achieving accuracy similar to full-precision models with fewer bits. For those who do not want to undertake QAT, Post-training Quantization (PTQ) is an alternative approach that requires only a small amount of unlabeled data for quantization. For PTQ, ACIQ \cite{banner:aciq} computes the optimal clipping range and the channel-wise, enabling the quantization of models into 4 bits without significant accuracy loss. OMSE \cite{choukroun:low} eliminates channel-wise activation quantization and proposes minimum mean-square error (MSE) quantization for greater accuracy improvements. 
DFQ \cite{nagel:data} presents a data-free PTQ framework that achieves accuracy comparable to full precision, while Bitsplit \cite{wang2020towards} relies on bit-split and stitching to obtain accurate quantized models. EasyQuant \cite{wu2020easyquant} introduces scale optimization to weights and activations, while AdaRound \cite{nagel:up} demonstrates that round-to-nearest is not the optimal solution and proposes an adaptive rounding method for weight quantization. 
AdaQuant \cite{hubara:accurate} proposes a general method that permits the adjustment of both quantization parameters and weights as required. Additionally, BRECQ \cite{li:brecq} demonstrates that block-wise minimum MSE optimization is the superior approach for PTQ. It employs STE training \cite{bengio2013estimating} for activation quantization and uses the Fisher Information Matrix to evaluate the importance of each layer in mixed precision quantization. 
QDrop \cite{wei2022qdrop} confirms that integrating activation quantization into PTQ leads to better final accuracy. OBC \cite{frantar2022optimal} proposes a new PTQ framework and model pruning based on the Optimal Brain Surgeon framework \cite{hassibi1993optimal}.

\subsection{Outliers Solutions}
In transformer-based language models, such as Opt \cite{zhang2022opt} and Bloom \cite{workshop2022bloom}, the presence of outliers poses challenges to the quantization process. To address this issue, a method called Outliers Suppression \cite{wei2022outlier} has been proposed. This method focuses on finding an optimal clip range for quantization, thereby mitigating the impact of outliers and pushing the limit of Post-training Quantization for BERT to 6-bit. Another approach, LLM.int8(), introduces a novel computing architecture that specifically targets outliers. By handling outliers separately, this method aims to improve the quantization process in transformer-based language models. SmoothQuant \cite{xiao2022smoothquant} is another technique proposed for LLMs. It introduces channel smooth parameters to balance the quantization difficulty between weight and activation. By effectively managing the quantization process, SmoothQuant achieves comparable accuracy to full precision models when using tensor-wise INT8 Post-training quantization on LLMs. These methods demonstrate innovative approaches to address the outliers' problem in transformer-based language models and push the limits of Post-training Quantization, resulting in improved quantization accuracy and efficiency for these models. 

The previous methods designed for addressing the outliers' issue in language models may not be directly applicable to reparameterized models. The nature of outliers in language models, which are token-specific, differs from the outliers observed in convolutional networks, which are sample-specific. In convolutional networks, each input contains only one sample and is processed independently. For instance, the Outliers Suppression method utilizes a token-wise clipping strategy for activation quantization in language models. However, employing such a strategy in convolutional neural networks would require dynamic quantization, which is often not hardware-friendly. If we were to use a sample-wise clipping strategy for CNN models, the quantization scheme for activation would become dynamic, necessitating the computation of feature distributions for each sample and layer. Unfortunately, obtaining these distributions for quantization purposes would require significant time and computational resources, potentially outweighing the time benefits gained from quantization. Therefore, it is essential to develop specialized techniques that are tailored to the characteristics of reparameterized models, considering the unique nature of activation outliers in these models.
\section{Problem Analysis}
For reparameterized models, the standard 8-bit PTQ method often leads to significant accuracy degradation. Previous studies such as RepOpt and QARepVGG have highlighted that outliers in these models are primarily caused by the BatchNorm layers. The running mean pushes the mean of distribution into 0 when the running variance increases the variance of distribution and directly affects the quantized accuracy \cite{chu2022make}. Notably, the use of pre-add BatchNorm, as opposed to post-add BatchNorm, proves more advantageous in preserving accuracy \cite{ding:re}. With pre-add BatchNorm, the BatchNorm modules do not affect the output, and the sum of branches remains unnormalized, potentially resulting in larger outliers. In the Appendix \ref{CompNorm}, we provide a comparison of the quantized accuracy and full precision accuracy for RepVGGA0 models with post-add BatchNorm and pre-add BatchNorm. The results demonstrate that post-add BatchNorm effectively suppresses outliers and allows for the quantization of reparameterized models with minimal accuracy loss.

\subsection{Outliers Analysis}
To gain deeper insights into the relationship between outliers and quantization performance, we conducted an analysis on RepVGGB0. Specifically, we focused on the first, fourth, and eighth blocks in stage 3 and examined the maximum and minimum values across different samples at Fig. \ref{fig_FM} (a)-(c). We observed that the maximum and minimum values vary across samples, indicating the presence of outliers. Interestingly, samples that exhibit outliers in the first block tend to have outliers in the subsequent blocks as well. However, the values of outliers differ among the layers. Additionally, when considering all samples with outliers, we found that the scaling of each sample is also different. For instance, the outliers of the 26th sample are amplified after several convolution layers, while the outliers of the 29th sample are narrowed down during the same convolution layers. This suggests that the model behaves differently for different samples.

To gain further insights into the outliers in different samples, we conducted an analysis by plotting the maximum values per channel for three specific samples (2nd, 26th, and 30th) in the same layer at Fig. \ref{fig_FM} (d)-(f). We observed that while the values of outliers differ, they consistently appear at fixed channels. Moreover, the ratios between outliers in each sample are similar, but the ratio of outliers to normal values varies. This poses a challenge for methods like SmoothQuant, which aims to balance the quantization difficulties between activation and weights by introducing scaling parameters. Since the ratios of outliers to regular values differ among samples, a general scaling parameter may not be suitable for all samples. Additionally, the fusion of these scaling parameters is only feasible when the activation is linear.
\begin{figure}[!t]
  \begin{center}
  \centerline{\includegraphics[width=1\columnwidth]{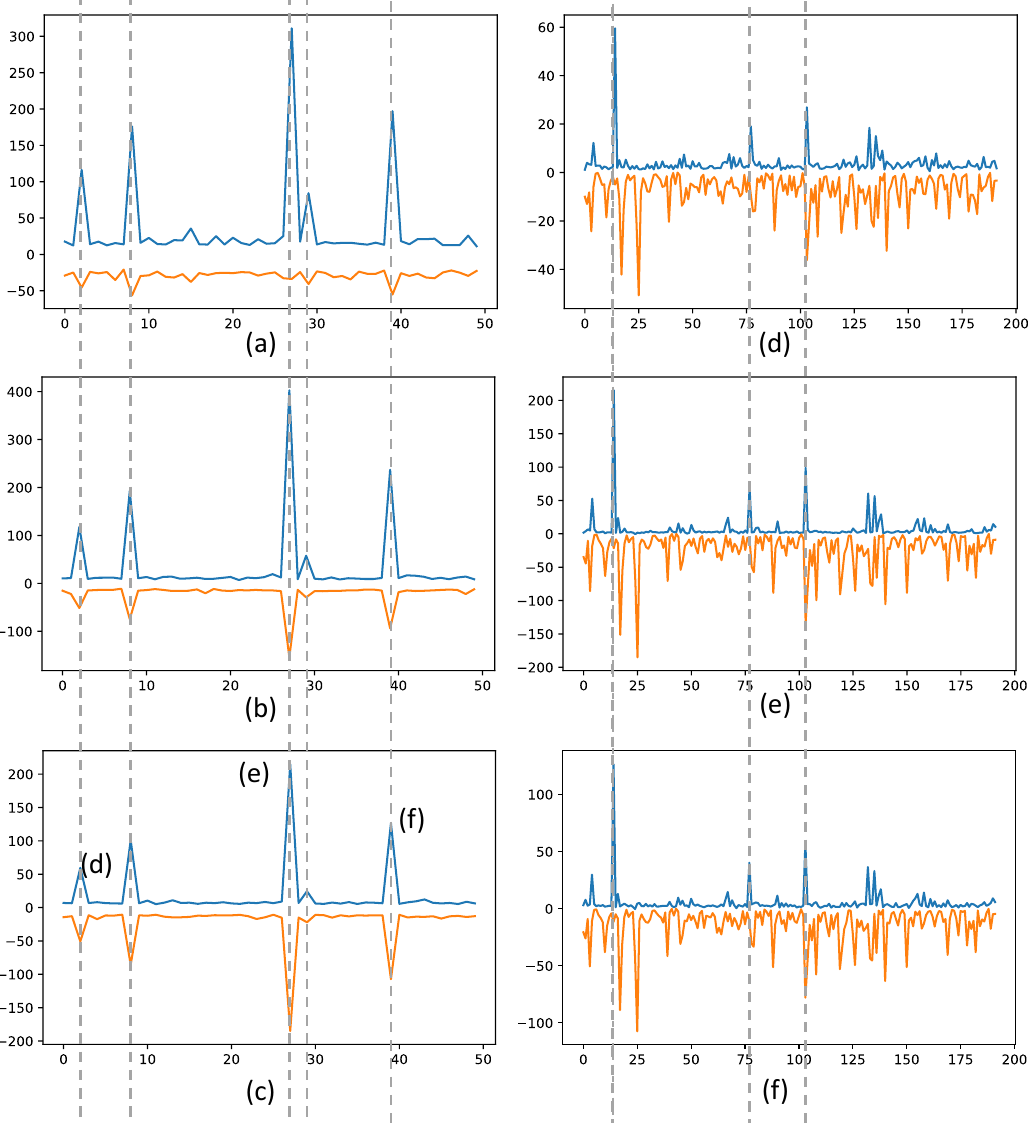}}
  \caption{(a) (b) (c) is maximums and minimums of the different samples in same layer. (d) (e) (f) is maximums and minimums of different channels in different samples of same layers. Orange presents the minimums and blue means the maximums.}
  \label{fig_FM}
  \end{center}
  \vskip -0.2in
\end{figure}

The presence of outliers in language models (LLMs) and convolutional networks (CNNs) manifests similarly, but there are important differences due to the input nature of these models. LLMs take token sequences as input, while CNNs only require a single input. This distinction allows the impact of hard-to-quantize tokens to be mitigated in LLMs when they are concatenated with easier-to-quantize tokens. Outliers Suppression proposes an aggressive clipping strategy that effectively disregards outliers of less relevant tokens and clips them to small values. The study found that even with a clipping ratio as high as 0.3\%, the accuracy of LLMs remains largely unaffected.

To highlight the differences between LLMs and CNNs, we conducted a comparative experiment on reparameterized convolution networks and NLP models, as depicted in 
Fig. \ref{fig_A0_Acc}. The results demonstrate that the accuracy of most reparameterized models deteriorates significantly with a 20\% clipping strategy, while RoBERTA\_QNLI only experiences a minor 0.5\% accuracy drop even with a 5\% clipping strategy (equivalent to a 20\% clipping strategy for most reparameterized models). There is a 2-bit expression gap observed in quantization. With more aggressive clipping strategies, reparameterized models suffer substantial accuracy losses even with a 1\% clipping strategy. Notably, MobileOneS0 loses representational ability under a 1\% clipping strategy. Unlike LLMs, outliers in reparameterized models can be suppressed to some extent, but they still have an impact on accuracy. Hard-to-quantize samples are particularly affected by aggressive clipping strategies. However, a milder clipping policy still results in significant outliers, leading to accuracy degradation in static quantization.

\begin{figure}[!tb]
  \begin{center}
  \centerline{\includegraphics[width=1\columnwidth]{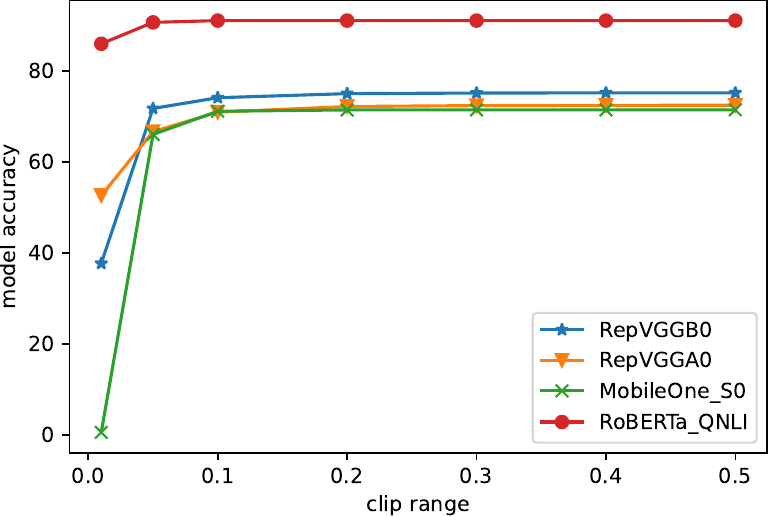}}
  \caption{To detect the impact of clipping the outliers, we set the 50 times bigger than mean as outliers. For outliers, we enumerate the value to cut the clip ratio to cut the activation on several reparameterized models. The clipping reuslts of RoBERTA is from Outlier Suppression \cite{wei2022outlier}.}
  \label{fig_A0_Acc}
  \end{center}
  \vskip -0.2in
\end{figure}

We extended our investigation to analyze the impact of different clipping strategies. Specifically, we employed a 0.5 clipping strategy and quantized models using the max-min quantization method. The resulting boxplot for RepVGGB0 is depicted in Fig. \ref{fig_range}. The findings revealed that the clipping strategy effectively limits the presence of outliers. Notably, setting a clipping range in a preceding layer can attenuate the outliers' phenomenon in subsequent layers. This observation suggests that the outliers in the output are often a consequence of outliers in the input. However, it should be noted that in certain layers, the application of clipping strategies can introduce outliers in the output.

\begin{proposition}
  For a given input that follows a Gaussian mixture distribution $X \sim \sum_{K=1}^{2} \alpha_{x_K}\mathcal{N}(\mu_{x_K}, \sigma_{x_K}^2)$ and a weight of a convolution layer that follows a Gaussian distribution $W \sim \mathcal{N}(\mu_{w}, \sigma_{w}^2)$, the output of the convolution can be obtained as $O = \text{Conv}(W, X)$. The output also follows a mixture distribution $O \sim \sum_{K=1}^{2} \alpha_{o_K}f_{o_K}(x)$, where $\sigma_{o_2} / \sigma_{o_1} < \sigma_{x_2} / \sigma_{x_1}$.
\end{proposition}
\begin{proof}
  Since the outliers are mainly affected by variance parameters, we can assume that the outliers are scalar errors \cite{dixon1953processing}. For computational simplicity, we set the means of each distribution in the Gaussian mixture distribution to be the same. Let $X$ follow the distribution p$\mathcal{N}(\mu_{x}, \sigma_{x}^2) + q\mathcal{N}(\mu_{x}, t^2\sigma_{x}^2)$. The probability density function $f_X(x)$ can be written as::
\begin{equation}
  \begin{aligned}
    f_X(x) &= pf_{X_1}(x)+qf_{X_2}(x) \\
    &= \frac{p}{\sigma_{x}\sqrt{2\pi}} e^{-\frac{(x-\mu_{x})^2}{2\sigma_{x}^2}} + \frac{q}{t\sigma_{x}\sqrt{2\pi}} e^{-\frac{(x-\mu_{x})^2}{2t^2\sigma_{x}^2}}
  \end{aligned}
\end{equation} 

Assuming $W$ conforms $N(0, \sigma_{w}^2)$. The size of input is $(C_{in}, H_{in}, K_{in})$ and the kernel size of the weight is $(C_{out}, C_{in}, K_{h}, K_{out})$. If we use the im2col method for convolution, the convolution operation can be expressed as a matrix multiplication. Let $\dot{X}_{(C_{in}, K_{h}\times K_{w}, H_{out}\times W_{out})}$ and $\dot{W}_{(C_{out}, C_{in}, K_{h}\times K_{w})}$ be sample matrices of $X$ and $W$.
The output of convolution can be calculated as follows:
\begin{equation}
  \begin{aligned}
    &O_{(C_{out}, H_{out} \times W_{out})} = \dot{W} \times \dot{X} = \sum\limits_{i=0}^{C_{in}}\sum\limits_{j=0}^{K_{h}\times K_{w}} O'\\
    &= \sum\limits_{i=0}^{C_{in}}\sum\limits_{j=0}^{K_{h}\times K_{w}} \begin{bmatrix}
        \begin{array}{cccc}
           w_{0ij}x_{ij0} & w_{0ij}x_{ij1} & \cdots & w_{0ij}x_{ijm} \\
           w_{1ij}x_{ij0} & w_{1ij}x_{ij1} & \cdots & w_{1ij}x_{ijm} \\
          \vdots & \vdots & \ddots & \vdots \\
          w_{nij}x_{ij0} & w_{nij}x_{ij1} & \cdots & w_{nij}x_{ijm}
        \end{array}
    \end{bmatrix}
  \end{aligned}
\end{equation}
As the $\dot{X}$ and $\dot{W}$ are sample matrices, $x \in X$ and $w \in  W$. For the matrix $O'$, the probability distribution function can be easily obtained as:
\begin{equation}
  \begin{aligned}
    f_O'(u) &= \int_{-\infty}^{+\infty}f_X(x)f_W(u/x) dx\\ 
    &= p\int_{-\infty}^{+\infty}f_{X_1}(x)f_W(u/x) dx + qf_{X_2}(x)f_W(u/x) dx \\ 
    &= pf_{O'_1}(x) + qf_{O'_2}(x)
  \end{aligned}
\end{equation}

The output can also be expressed in the form of a mixture distribution. We can calculate the variances for each component as:

\begin{equation}
  \begin{aligned}
    Var(O'_{1}) &= (\sigma_{x}^2 + \mu_{x}^2)(\sigma_{w}^2 + \mu_{w}^2) - \mu_{x}^2\mu_{w}^2 \\
    Var(O'_{2}) &= (t^2\sigma_{x}^2 + \mu_{x}^2)(\sigma_{w}^2 + \mu_{w}^2) - \mu_{x}^2\mu_{w}^2
  \end{aligned}
\end{equation}
From this, we can easily calculate the ratio of variances as:
\begin{equation}
  \begin{aligned}
    \frac{Var(O'_{2})}{Var(O'_{1})} &= \frac{t^2\sigma_{x}^2\sigma_{w}^2 + t^2\sigma_{x}^2\mu_{w}^2 + \mu_{x}^2\sigma_{w}^2}{\sigma_{x}^2\sigma_{w}^2 + \sigma_{x}^2\mu_{w}^2 + \mu_{x}^2\sigma_{w}^2} \\
    &= t^2 + \frac{\mu_{x}^2\sigma_{w}^2-t^2\mu_{x}^2\sigma_{w}^2}{\sigma_{x}^2\sigma_{w}^2 + \sigma_{x}^2\mu_{w}^2 + \mu_{x}^2\sigma_{w}^2}
  \end{aligned}
\end{equation}
Since $t > 1$, we have $1-t^2<0$. Thus, the ratio of variances is smaller than $t^2$.
\end{proof}

In Fig. \ref{fig_range} (a), it can be observed that unless a layer introduces outliers itself, the outliers of each subsequent layer tend to reduce sequentially. Taking the layers in stage 1 (2nd - 5th layers) as an example, the outliers show a decreasing trend. 

\begin{figure}[!b]
  \centering
  \subfloat[full precision activation range]{
      \includegraphics[width=1\columnwidth]{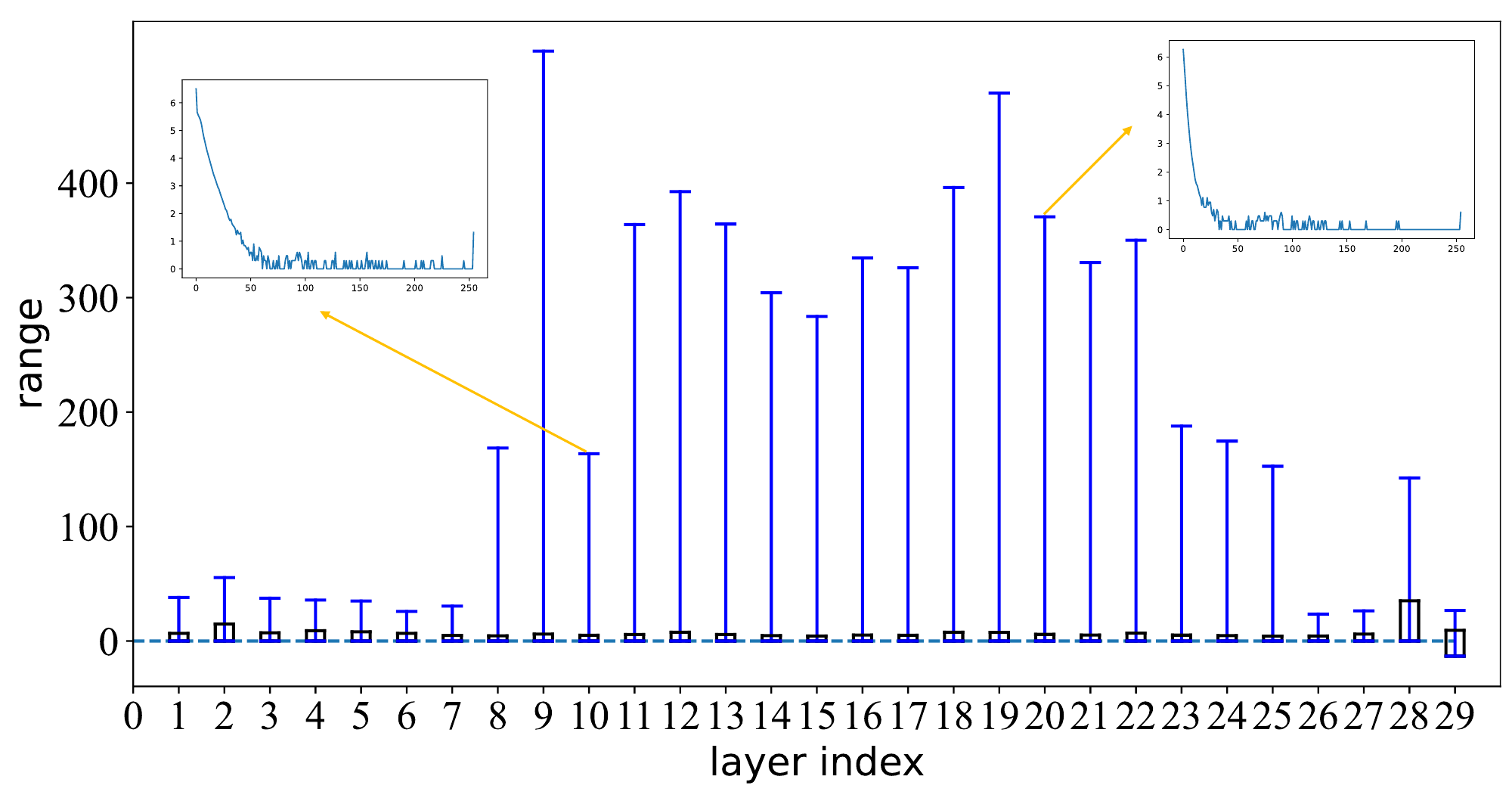}
  }
  \quad
  \subfloat[activation range quantized with 0.5 times maximums]{
      \includegraphics[width=1\columnwidth]{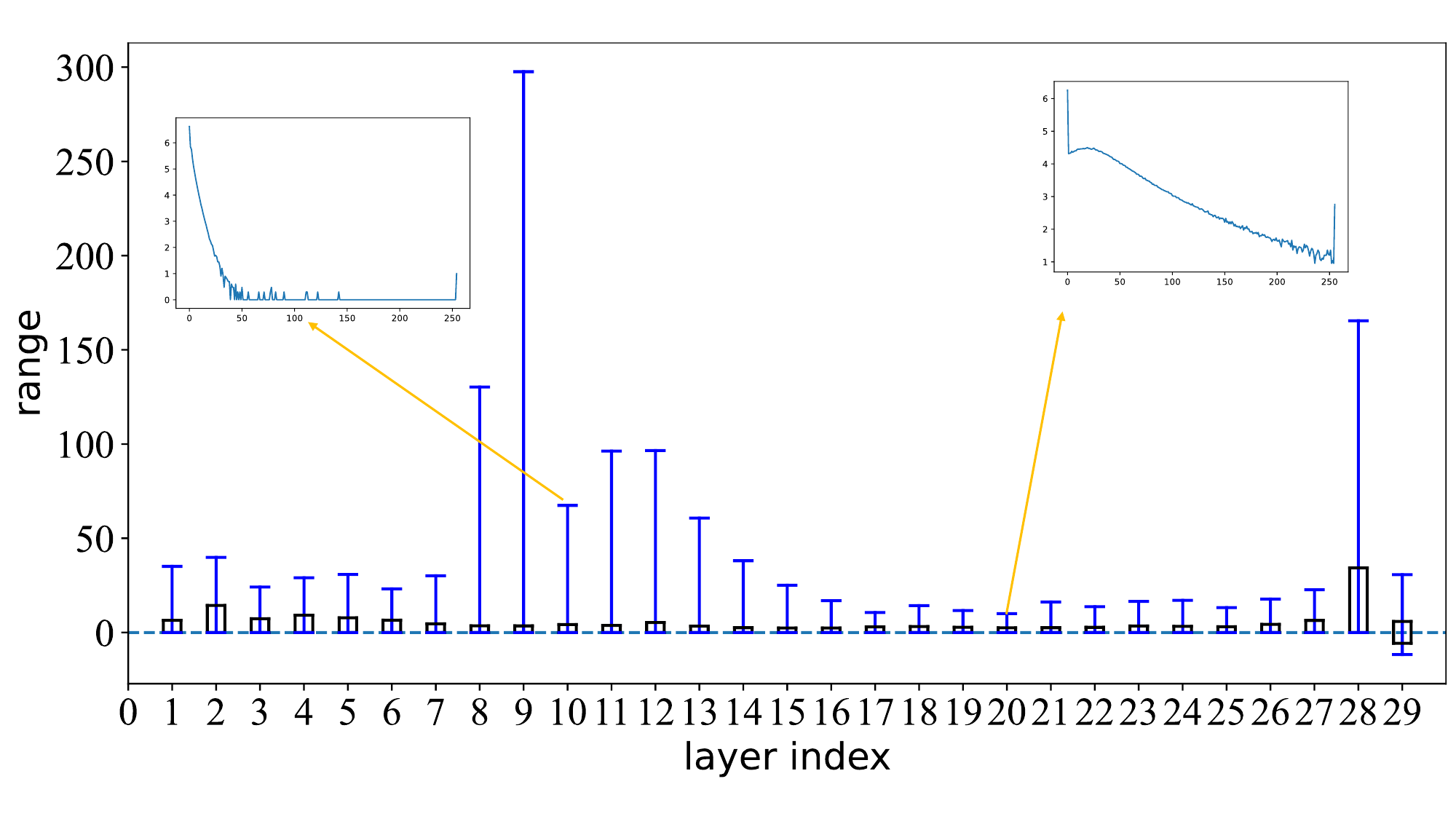}
  }
  \quad
  \centering
  \caption{Boxplot of full precision activation range and initial quantized activation range. We increase the range of boxes from 75\% to 99\%. For the activation range, we do not count the number of 0 values, because its size does not affect quantization. Subplots are the distribution of quantized activation for the 10th and 20th layers. The quantized activation is quantized with initial quantization parameters. }
  \label{fig_range}    
\end{figure}

\subsection{quantization measuremant analysis}
\label{3.b}
  For optimal quantization of signals, Mean Squared Quantization Error (MSQE) is proposed to measure the optimal quantization scale. It is defined as follows:
  \begin{equation}
    D(x,\hat{x}) = \int_{0}^{+\infty} (x-\hat{x})^p f(x)dx
  \end{equation}
  Here, $f(x)$  represents the probability density function of the original signal. In the case of outliers introduced by the affine parameters of BatchNorm, we can model the feature map with outliers using a Gaussian mixture model:
  \begin{equation}
    f_{act}(x) = \frac{1}{N}\sum_{n=1}^{N} \mathcal{N}(\mu, t^2_n \sigma^2)
  \end{equation}
  As discussed earlier, in several channels, $t_N$ takes on large values. For most channels, $t_N$ is similar, allowing us to treat these channels as a collective Gaussian distribution. In regression problems, Mean Squared Error (MSE) is not suitable for data with outliers that are not crucial for the results. MSE amplifies the distortion between predicted values and real values for outliers, leading to deviations from the optimal position in the output. The outliers in the activation have a similar effect on quantization scales.
  \begin{proposition}
    For a Gaussian distribution followed $\mathcal{N}(\mu, \sigma^2)$, assume $h$ is the optimal accuracy quantization clipping value for this distribution. For a Gaussian mixture distribution with scalar outliers $\sum_{n=1}^{N} \mathcal{N}(\mu, t_n^2\sigma^2)$, using a L$p$ Loss as distortion for quantization. The best point for the quantization is $\sqrt[p]{\sum_{0}^{N}t_n^p/N}h$.
  \end{proposition}

\begin{proof}
  Let us assume $X \sim \mathcal{N}(0, t_n^2\sigma^2)$, and consider the distortion $D(X, \hat{X})$ for quantization. We can express it as follows:
  \begin{equation}
    \begin{aligned}
      D(X, \hat{X}) &= \frac{1}{N} \sum_{n=0}^{N} \int_{0}^{+\infty} |x - \hat{x}|^p f_n(x)dx \\
      &= \int_{0}^{+\infty} \frac{|t_nx - \widehat{t_nx}|^p}{N}f(x)dx \\
      &\approx \int_{0}^{+\infty} \sqrt[p]{\sum_{0}^{N}\frac{t_n^p}{N}}|x-\hat{x}|^p f(x) dx \\
    \end{aligned}
  \end{equation}
  Here, $f(x)$ represents the probability density function for a standard Gaussian distribution. We can rewrite this expression as:
  \begin{equation}
    \begin{aligned}
      D(X, \hat{X}) &= \int_{0}^{+\infty} \sqrt[p]{\sum_{n=0}^{N}\frac{t_n^p}{N}}|x-\hat{x}|^p \frac{1}{\sqrt{2\pi}\sigma} e^{-\frac{x^2}{2\sigma^2}} dx \\
      &= \int_{0}^{+\infty} |x-\hat{x}|^p \mathcal{N}(\mu, k^2\sigma^2)dx \\
      k &= \sqrt[p]{\sum_{n=0}^{N}\frac{t_n^p}{N}} \\
    \end{aligned}
  \end{equation}
  Thus, we can interpret this as a quantization distortion for a Gaussian distribution with parameters $\mathcal{N}(\mu, k^2\sigma^2)$, where the optimal quantization clipping value is given by $kh$.
\end{proof}

The scalar factor $k=\sqrt{\sum_{n=0}^{N}t_n^2 / N}$ in the MSE represents the mean square of the scalar factors for each variance $t$. However, the mean square amplifies the effect of outliers, giving them more influence than normal values. To mitigate this, we can reduce the value of $p$. By choosing $p$ smaller, we can preserve the effect of outliers. In our papers, we utilize the Mean Absolute Error (MAE) instead of the MSE. In MAE, the scalar factor is calculated as $k=\sum_{n=0}^{N}t_n/ N$, which represents the mean of the scalar factors for each variance $t$. Each Gaussian distribution contributes to the optimal quantization due to its expected value. Moreover, the optimal quantization clipping values obtained using MAE and MSE are close and can be considered similar. However, MAE has the advantage of minimizing the impact of outliers and being able to focus more on the important parts without outliers compared to MSE.

\section{Method}
\label{method}
In this section, we propose RepAPQ, an accurate Post-training Quantization framework for reparameterization networks. The framework, depicted in Fig. \ref{fig_rep}, comprises two methods: Quantization Protect Reparameterization and Block-Across Calibration. 

In RepAPQ, we partition the neural network into stages. Each stage, denoted by $n$, consists of $N$ blocks. We aim to distill the quantized blocks using the output of the full precision network. For reparameterized models, these blocks correspond to reparameterized convolution layers with activation layers. To strengthen the relationship between the blocks, we incorporate the distortion of the stage output. Within each block, we introduce Quantization Protecting Reparameterization, which inserts an affine layer after the reparameterized convolution layers. These affine layers serve two purposes: they facilitate faster convergence during optimization, and the channel-wise affine layer accommodates channel-specific outliers, adjusting the output to be closer to the real output.

\subsection{Preliminaries}

In our method, we followed the original quantization method for weight $w$.
\begin{eqnarray}
    \widehat{w} &=& clamp( \lfloor \frac{w}{s_{w}} \rceil, -2^{b-1}, 2^{b-1}) \times s_{w}
\end{eqnarray}
Here, $s_{w}$ represents the quantization scale for the weights, and $b$ denotes the desired quantization bit-width.

For activation quantization, we utilize the BatchQuant approach \cite{bai2021batchquant}. BatchQuant utilizes a moving average to calculate the extreme values, $x_{max}$ and $x_{min}$, for each layer of the neural network based on the current batch. It also incorporates a learnable scaling parameter, denoted by $\eta$, to normalize the activations. Additionally, a learnable parameter, $\epsilon$, is introduced to minimize quantization errors, similar to the approach used in LSQ+ \cite{bhalgat2020lsq+}.
The BatchQuant procedure for activation quantization can be described as follows:
\begin{equation}
  \begin{aligned}
    s_{x} = \frac{x_{max}-x_{min}}{2^{n}-1}\eta, \beta_{x} = \lfloor -\frac{x_{min}}{s_{x}}  - \epsilon \rceil \\
    \widehat{x} = ((\lfloor \frac{x}{s_{x}} \rceil + \beta_{x}).clamp(0, 2^b-1) - \beta_{x})s_{x}
  \end{aligned}
\end{equation} 
In the above equations, $s_x$ represents the scaling factor for activation, and $\beta_x$ denotes the zero point of activation. These values are calculated based on the extreme values obtained during the calibration process. It is important to note that the extreme values are fixed to ensure consistency throughout the calibration process.

\begin{figure*}[!t]
    \begin{center}
    \centerline{\includegraphics[width=1\linewidth]{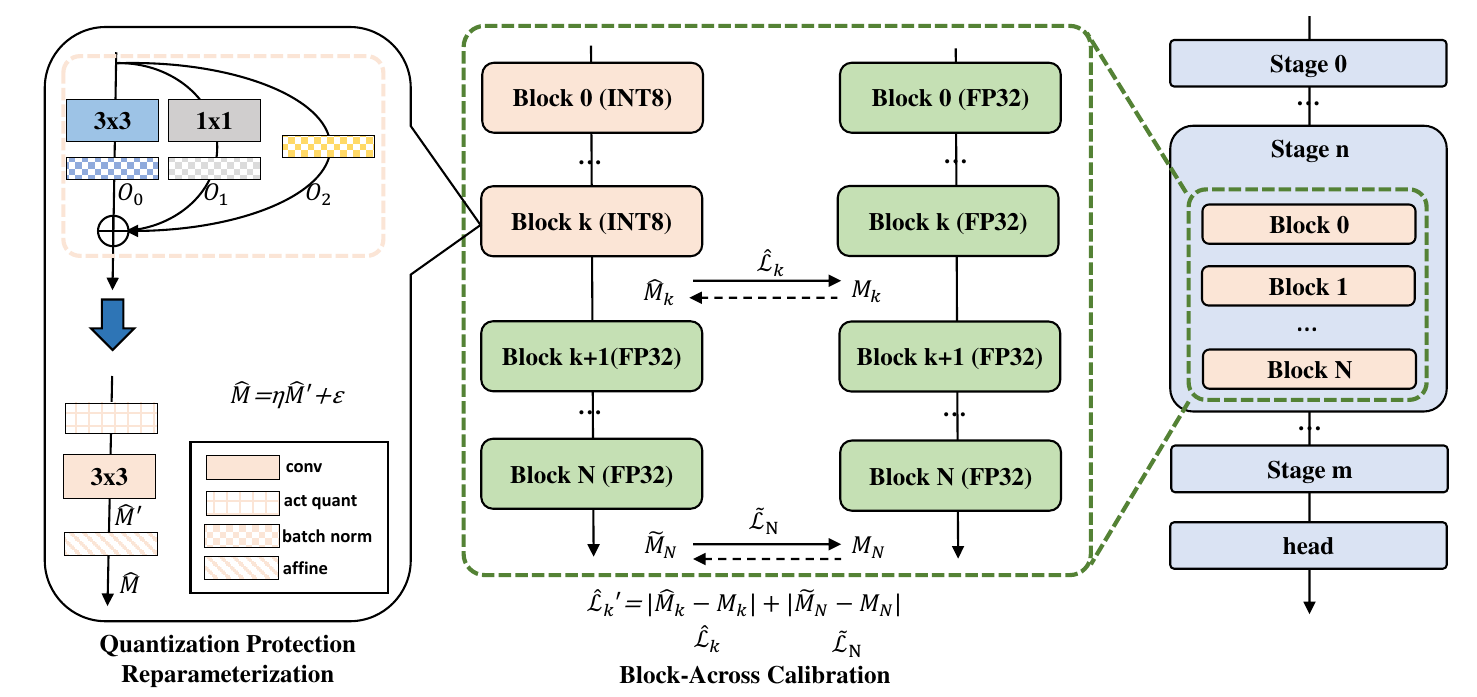}}
    \caption{The process of our method.  We consider a neural network with stage $n$ consisting of $N$ blocks. For the ABC technique, we utilize the feature maps of intermediate blocks and the stage to distill the quantized block. We only use the output of the distillated block, and as we progress deeper into the blocks, we strengthen the loss function. Regarding the QPRep technique, we first merge the three branches and introduce an affine layer. During inference, the affine parameters are combined with the quantization scales and convolution biases.}
    \label{fig_rep}
    \end{center}
    \vskip -0.2in
\end{figure*}

For the AdaQuant scheme, our objective is to minimize the mean squared error (MSE) between the per-layer output and its corresponding quantized output. AdaQuant leverages the quantized values as inputs for sequential quantization, which helps improve the overall accuracy.

Mathematically, we can express the optimization objective as follows:
\begin{equation}
  \begin{aligned}
    (s_{w}, s_{x}, V)_v &= \mathop{argmin}\limits_{s_{w}, s_{x}, V} | M_{k} - \widehat{M_{k}} |^{2} \\
    &= \mathop{argmin}\limits_{s_{w}, s_{x}, V} |L_k(L_{k-1}(\dots L_0(X))) \\
    &- \widehat{L}_k(\widehat{L}_{k-1}(\dots \widehat{L}_{0}(X)))|^{2}
  \end{aligned}
\end{equation}
Here, $M_{k}$ denotes the output of the $k$-th layer, and $\widehat{M_{k}}$ represents the quantized output of the $k$-th layer. $(s_{w}, s_{x}, V)_{k}$ refers to the quantization parameters for the $k$-th layer, and $V$ represents the updated value of the weight $w$.

Based on Proposition 2, we choose the mean absolute error (MAE) as the distortion metric for both quantized values and full precision values in our method. Additionally, in this section, we utilize the Gaussian mixture model proposed in Section \ref{3.b} to analyze the gradient, considering the different distortion metrics for quantization and full precision values.

\subsection{Quantization Protecting Reparameterization}
According to Proposition 1, when convolution layers receive inputs with outliers, the output of these layers also contains outliers. However, by applying quantization clipping ranges, we can limit the presence of outliers. In sequential quantization, we utilize the quantized values as inputs to distillate the layer to be quantized. This process further restricts the outliers in the output. As we use the output of each block to distillate the quantized output, the outliers in the full precision output are larger compared to the quantized output.
Assuming we have an $n$-th channel with outliers, the gradient with respect to $\widehat{M}_{n}$ can be expressed as follows:
\begin{equation}
  \frac{\partial Loss}{\partial \widehat{M_{n}}} = sgn(\widehat{M_{n}} - M_{n}) \\
\end{equation}
Approximating the gradient, we have:
\begin{equation}
  \begin{aligned}
    \partial \widehat{M_{n}} &\approx \int_{\widehat{M_{n}}_{max}}^{+\infty} sgn(M-\widehat{M_{n}}_{max})f_{M_{n}}(x) dx \\
    &\approx 1-\varGamma(\frac{\widehat{M_{n}}_{max} - \mu}{t_n\sigma})
  \end{aligned}
\end{equation}  
Here, $\widehat{M}$ still exhibits an increasing trend, but the rate of increase is limited and will not become excessively large.

Taking $s_w$ as an example, we can compute the gradient of$s_w$ as follows:
\begin{equation}
  \begin{aligned}
    \frac{\partial Loss}{\partial s_{w}} &= \frac{\partial Loss}{\partial \widehat{M}} \frac{\partial \widehat{M}}{\partial \widehat{w}} \frac{\partial \widehat{w}}{\partial s_{w}} \\ &= sgn(\widehat{M}-M)\widehat{x}(\frac{w}{s_{w}} - \lfloor \frac{w}{s_{w}} \rceil)
  \end{aligned}
\end{equation}
Similar to the $\widehat{M}$,  the magnitude of the gradient is limited, and the value of $s_{w}$ will not increase excessively, thereby avoiding the introduction of a poor quantization scale. This observation can be generalized to all optimizable values, indicating that MAE can effectively limit the introduction of outliers.

To account for samples with outliers, we introduce a new degree of freedom to amplify the output of each layer, bringing it closer to the original value of the outliers. During Reparameterization, we insert a channel-wise affine layer after the reparameterized convolutional layers:
\begin{equation}
  \widehat{M} = \eta \widehat{M'} + \varepsilon 
\end{equation}
Here, $\widehat{M'}$ represents the output of the reparameterized convolutional layers, $\eta$ is a channel-wise learnable parameter, and 
$\varepsilon$ is a bias term. This affine layer enables the adjustment of the output to better capture the channel-specific outliers and bring it closer to the true output.

The gradient of $\eta$ can be computed as follows:
\begin{equation}
  \frac{\partial Loss}{\partial \eta} = \frac{\partial Loss}{\partial \widehat{M'}} \frac{\partial \widehat{M'}}{\partial \eta} = sgn(\widehat{M}-M)\widehat{M'}
\end{equation}
The gradient of $\eta$ is relative to $\widehat{M'}$. Since $\widehat{M'}$ may have outliers, $\eta$ will learn to fit the distribution of $\widehat{M'}$. Assuming $n$-th channel  has outliers, we can calculate the gradient of $\eta_{n}$ and integrate it over the range of $\widehat{M_{n}'}$ as follows:
\begin{equation}
  \begin{aligned}
    \partial \eta_{n} &= \int_{0}^{+\infty} sgn(\widehat{M_{n}}-M_{n})\widehat{M_n'} dM_n' \\
                   &\approx \int_{max(\widehat{M_n'})}^{+\infty} sgn(\widehat{M_n}-M_n)\widehat{M_n'} f_{M_n'}(M_n')dM_n'
  \end{aligned}
\end{equation}
Here, $f_{M_n'}(M_n')$ represents the probability density function of $\widehat{M_n'}$. By integrating over the range of $\widehat{M_n'}$ starting from the maximum value, we take into account the contribution of outliers and their impact on the gradient of $\eta_n$.

For channels with outliers, the gradient of $\eta_n$ is larger, leading to an increase in $\eta_n$ itself. A large value of $\eta_n$ amplifies the output of the channel, bringing the outliers closer to their original values. However, the use of MAE helps prevent $\eta$ from becoming excessively large, thereby preserving the overall distribution of the output. Furthermore, a larger $\eta$ facilitates faster convergence of $\widehat{M}$ towards the optimal position during training. The insertion of affine layers aids in accelerating the convergence of quantization parameters, preventing them from growing too large. During inference, the affine parameters can be fused into the quantization scale and convolution bias, thereby incorporating their effects without introducing additional accuracy loss or computational overhead \cite{hubara:accurate}. Typically, quantization scale and convolution bias are represented with a specific bit-width, such as 32 bits or 16 bits, which corresponds to the precision of the multiply-add operation.

\subsection{Across-block Calibration}
While MAE introduces more quantization noise compared to MSE due to its element-wise gradient, in practice, it is beneficial to incorporate an element-wise gradient for more accurate gradient estimation. Block reconstruction techniques, such as BRECQ \cite{li:brecq}, leverage the output of blocks to distillate the quantized block output, effectively mitigating overfitting in quantized models. In reparameterized models, blocks consist of a convolution layer and an activation layer, deviating from the structure of traditional neural networks.
Moreover, stage reconstruction is generally less effective than block reconstruction, as discussed in BRECQ. Traditional model designs often involve a larger number of stages, and optimizing them jointly can lead to underfitting. In reparameterized models, blocks are fused during inference, and the reconstruction of blocks simplifies layer reconstruction.

When reconstructing the 3-4 convolutional layers in the model into a "block" following the structure of other networks, there is a possibility that layers with outliers may be selected as the output of the "block". According to Proposition 1, after a layer whose weights follow a Gaussian distribution without outliers, the outliers in the feature map are limited.

In the optimization process, despite the outlier issue, we continue to utilize MAE as the measurement of the block output. Specifically, for the $k$-th block within a stage consisting of $m$ blocks, our objective is to minimize the discrepancy between the output of the $k$-th block and its estimate, as well as the discrepancy between the output of the entire stage and its estimate. The optimization problem can be formulated as follows:

\begin{equation}
  (s_{w}, s_{x}, V)_{k} = \mathop{argmin}\limits_{(s_{w}, s_{x}, V)_{k}} | \widetilde{M}_{N} - M_{N} | + | \widehat{M}_{k} - M_{k} |
\end{equation}

Where $\widetilde{M}_{N}$ and $M_{N}$ represent the output of the stage and its estimate, respectively. During the distillation of the $k$-th block, the other blocks within the stage are kept fixed.

To compute the gradient of the loss with respect to $\widetilde{M}_{k}$, we can utilize the chain rule:

\begin{equation}
  \frac{\partial Loss}{\partial \widetilde{M}_{k}} = 1 + \frac{\partial \widetilde{M}_{k+1}}{\partial \widetilde{M}_{k}} \cdot \frac{\partial \widetilde{M}_{k+2}}{\partial \widetilde{M}_{k+1}} \cdots \frac{\partial \widetilde{M}_{N}}{\partial \widetilde{M}_{N-1}} = 1 + \prod_{n=k+1}^{N} w_n 
\end{equation}

where $\widetilde{M}_{N}$ is computed as the result of applying the network layers $L_{N}$, $L_{N-1}$, ..., $L_{k+1}$ to the output of the reparameterized layer $\hat{L}_{k}$, which in turn takes the input $\hat{x}$.

To ensure a more accurate gradient, we introduce an element-wise gradient for the output, taking the weights after the current blocks into account. The second part of the loss serves as input distillation for the blocks that follow the $k$-th block in the neural network architecture. This approach addresses the issue of consistent gradients for all elements and introduces the optimal stage gradient to expedite convergence.

In cases where a stage consists of only one block or when considering the last block in the stage, we still employ Mean Squared Error (MSE) as the measurement of block output. This is done to avoid introducing excessive noise by using Mean Absolute Error (MAE). The size of the calibration dataset and the number of iterations are related to the optimal number of parameters, as discussed in AdaQuant. To prevent overfitting and combine the measurement of layer output and stage output in our approach, we constrain the optimization to update only one block parameter every 1000 iterations.

\begin{table*}[!t]
  \caption{Top-1 Accuracy of 8-bit quantized repmodels compared with previous methods. The full precision model is from the open-source code. 
  For 8-bit quantization, all weight and input were quantized to 8bit.
  For 6-bit quantization, we keep the first and last layer as 8bits. $\dagger$ means replace the MSE with MAE.
  * present using the quantization scheme from QDrop papers. Full connected layers are channel-wise and all weight are asymmetric quantized.}
  \begin{center}

        \begin{tabular}{ccccccccccc}
          \hline
          Methods & \makecell[c]{Bits\\(W-A)} & \makecell[c]{Rep\\VGGA0} & \makecell[c]{Rep\\VGGA1} & \makecell[c]{Rep\\VGGA2} & \makecell[c]{Rep\\VGGB0} & \makecell[c]{Rep\\VGGB1} & \makecell[c]{Mobile\\OneS0} & \makecell[c]{Mobile\\OneS1} &\makecell[c]{Mobile\\OneS2} & Average \\
          \hline 
          Original & 32-32 & 72.41 & 74.46 & 76.46 & 75.15 & 78.37 & 71.42 & 75.94 & 77.43 & 75.21(-0.00) \\
          \hline
          W/O Rep & 8-8 & 71.75 & 73.99 & 76.35 & 75.18 & 74.63 & 70.26 & 74.95 & 76.64 & 74.22(-0.99)\\
          MIN/MAX & 8-8 & 34.34 & 57.33 & 60.06 & 0.49 & 0.55 & 68.23 & 72.63 & 75.04 & 46.08(-29.13)\\
          OMSE\cite{choukroun:low} & 8-8 & 50.31 & 59.37 & 57.86 & 36.67 & 7.53 & 67.24 & 72.60 & 75.33 & 53.46(-21.85) \\
          OMSE$\dagger$ & 8-8 & 51.05 & 60.85 & 60.85 & 54.73 & 31.19 & 68.45 & 73.25 & 75.66 & 59.50(-15.71) \\
          SmoothQuant\cite{xiao2022smoothquant} & 8-8 & 56.33 & 66.67 & 69.01 & 62.72 & 72.73 & 68.16 & 72.52 & 73.48 & 67.70(-7.51) \\
          Outliers Suppresion\cite{wei2022outlier} & 8-8 & 56.34 & 65.95 & 67.68 & 64.98 & 72.62 & 68.73 & 72.80 & 74.91 & 68.00(-7.21)\\
          AdaQuant\cite{hubara:accurate} & 8-8 & 70.22 & 72.46 & 75.14 & 73.36 & 76.80 & 69.89 & 72.86 & 75.57 & 73.29(-1.91)\\
          QDrop*\cite{wei2022qdrop} & 8-8 & 68.72 & 66.75 & 73.04 & 32.04 & 1.37 & 70.03 & 73.53 & 75.41 & 57.61(-17.60) \\
          \textbf{Ours} & 8-8 & \textbf{72.04} & \textbf{74.28} & \textbf{76.33} & \textbf{74.93} & \textbf{78.19} & \textbf{71.06} & \textbf{75.27} & \textbf{77.06} & \textbf{74.90(-0.31)} \\
          \hline
          AdaQuant\cite{hubara:accurate} & 6-6 & 67.05 & 67.42 & 71.51 & 68.74 & 72.64 & 65.84 & 69.10 & 70.63 & 69.12(-6.09) \\
          QDRop*\cite{wei2022qdrop} & 6-6 & 63.18 & 67.32 & 71.49 & 65.19 & 71.28 & 62.43 & 69.03 & 69.82 & 67.47(-7.74) \\
          \textbf{Ours} & 6-6 & \textbf{69.85} & \textbf{72.49} & \textbf{75.22} & \textbf{73.06} & \textbf{76.63} & \textbf{66.36} & \textbf{69.69} & \textbf{71.25} & \textbf{71.82(-3.39)} \\
          \hline
        \end{tabular}
  \end{center}
  \label{IC_exp}
  \vskip -0.2in
\end{table*}

\section{Experiments}

In this section, we present the application of our methods to various models and datasets. We initially focused on image classification models, specifically RepVGG \cite{ding:repvgg} and MobileOne \cite{vasu:improved}, which proved challenging to quantize even with 8-bit precision. Subsequently, we extended our experiments to object detection models using YoLov6 on the COCO dataset.

In all our experiments, we utilized a small calibration set randomly sampled from the full training dataset. The calibration set consisted of 1024 samples for ImageNet \cite{deng:imagenet} and 512 samples for COCO \cite{lin2014microsoft}. For each reparameterized layer, we performed 1000 iterations using the Adam \cite{kingma2014adam} optimizer and employed a CosineDecay learning rate scheduler to optimize the models.

\begin{table}[!t]
  \caption{Learning Rate for different models and bit-widths.}
  \begin{center}
  \begin{tabular}{lccccr}
  \hline
  \makecell[c]{Parameters} & \makecell[c]{RepVGG \\ A8W8/A6W6} & \makecell[c]{RepVGG \\ A4W4} & \makecell[c]{MobileOne\\ A8W8/A6W6}\\
  \hline
  \makecell[c]{Weight} & 1e-5 & 1e-4 & 1e-5 \\
  \makecell[c]{bias} & 1e-4 & 1e-3 & 1e-4 \\ 
  \makecell[c]{$s_{w}$} & 1e-5 & 1e-5 & 1e-3 \\
  \makecell[c]{$s_{x}$} & 1e-3 & 1e-3 & 1e-3 \\
  \makecell[c]{BN} & 1e-2 & 1e-1 & 1e-2 \\
  \makecell[c]{Others} & 1e-5 & 1e-5 & 1e-5 \\
  \hline
  \end{tabular}
  \end{center}
  \label{Exp_Setup}
  \vskip -0.2in
\end{table}
To ensure consistency, we set the learning rates of the parameters according to the values specified in Table \ref{Exp_Setup}. The remaining hyperparameters are set to their default values in PyTorch.

To simulate real deployment scenarios, we adopt specific quantization configurations. Specifically, we use tensor-wise asymmetric quantization with offset for activation quantization, and channel-wise symmetric quantization without offset for weight quantization. For efficient inference, we utilize tensor-wise symmetric quantization for weight quantization in fully connected layers.

\subsection{Results on Image Classification}
We present our results and provide a comparative analysis with previous methods in Table \ref{IC_exp}. Reparameterized models present a challenge for quantization into 8 bits, as it can result in significant accuracy loss. However, direct quantization without reparameterization can result in increased inference time and storage requirements. By changing the measurement from OMSE \cite{choukroun:low} to MAE, the results of 8-bit quantization are improved. 
Existing methods such as SmoothQuant\cite{xiao2022smoothquant} and Outliers Suppression\cite{wei2022outlier} can mitigate accuracy degradation, but they still fall short in terms of achieving comparable accuracy to full-precision models. Even the state-of-the-art method QDrop\cite{wei2022qdrop}, which utilizes asymmetric weight quantization to improve accuracy, can result in over 1\% accuracy loss on reparameterized models.
While vanilla AdaQuant can achieve around 1\% accuracy loss on RepVGGAs, the accuracy loss is higher for RepVGGBs. Our proposed method achieves less than 0.5\% accuracy loss for all RepVGG models and outperforms other methods for lighter models like MobileOne. It is worth noting that larger reparameterized models are more sensitive to quantization due to having more and larger outliers. Despite the design of AdaRound-like quantization being more favorable for low-precision quantization, our method surpasses previous approaches even for lower-bit quantization scenarios. In the context of 6-bit quantization, our method exhibits a notable improvement of nearly 3\% in accuracy when compared to other methods. This finding further underscores the effectiveness of our approach.

\subsection{Comparison with RepOpt and QARepVGG}

The primary goal of RepOpt is to incorporate reparameterization during training and directly obtain reparameterized models. However, RepOpt may not perform well with small network architectures like RepVGGA0 and can even lead to reduced accuracy. Additionally, RepOpt is not compatible with models that use different reparameterization methods, such as MobileOne. We perform a comparative analysis between the data mentioned in \cite{chu2022make} and our data in Table \ref{Opt_exp}. We find that while QARepVGG can partially address the quantization issue, it introduces additional accuracy degradation to original models. On the other hand, our quantized RepVGGA0 and RepVGGB0 models outperform the quantized QARepVGG models obtained using the QAT method. It's important to note that RepOpt and QARepVGG are training methods that require a significant amount of time to obtain a quantization-friendly model. Furthermore, they are model-specific, and designing reparameterized optimizers or losses for future models can be time-consuming. In contrast, our approach tackles the quantization problem from a quantization perspective and offers a model-agnostic method that is independent of RepOpt and QARepVGG. This allows for a more versatile and efficient solution that can be applied across different models without the need for model-specific modifications.
\begin{table}[!t]
  \caption{Comparison of 8-bit Quantized RepVGG Models: Our Method vs. RepOpt and QARepVGG. '$^\dagger$' indicates results from the original paper by QARepVGG \cite{chu2022make}.}
  \begin{center}
  \begin{tabular}{lcccr}
  \hline
  \makecell[c]{Model} & \makecell[c]{Accuracy\\FP32} & \makecell[c]{Accuracy\\PTQ} & \makecell[c]{Accuracy \\ QAT} \\
  \hline
  \makecell[c]{RepVGGA0+Ours} & 72.4 & \textbf{72.0(-0.4)} & / \\
  \makecell[c]{RepOptVGGA0$^\dagger$} & 70.9 & 64.8(-6.1) & / \\ 
  \makecell[c]{QARepVGGA0$^\dagger$} & 72.2 & 70.4(-1.8) & 71.9(-0.3) \\
  \makecell[c]{RepVGGB0+Ours} & 75.2 & \textbf{74.9(-0.3)} & / \\
  \makecell[c]{RepOptVGGB0$^\dagger$} & 73.8 & 62.6(-11.2) & / \\
  \makecell[c]{QARepVGGB0$^\dagger$} & 74.8 & 72.9(-1.9) & 74.6(-0.2) \\
  \makecell[c]{RepVGGB1+Ours} & 78.4 & \textbf{78.2(-0.2)} & / \\
  \makecell[c]{RepOptVGGB1$^\dagger$} & 78.5 & 75.9(-2.6) & / \\
  \makecell[c]{QARepVGGB1$^\dagger$} & 78.0 & 76.4(-1.6) & / \\
  \hline
  \end{tabular}
  \end{center}
  \label{Opt_exp}
  \vskip -0.2in
\end{table}

\subsection{Results on Object Detection}
Yolov6 utilizes RepVGGblock for enhanced feature extraction, but the reparameterization modules introduce challenges for quantization. To address this, the model has been redesigned to be more quantization-friendly. In our experiments, we evaluate both the original and redesigned models, and we have documented the results in Table \ref{Detect_exp}.
Our proposed method, as well as AdaQuant, achieves similar accuracy when applied to the quantization-friendly model. This indicates that our method is effective in preserving accuracy during quantization. However, the quantization challenges of Yolov6 are evident in the results of Quantization-Aware Training (QAT), where Yolov6s achieves only 43.3\% accuracy.
For Yolov6v1, our method outperforms AdaQuant and achieves nearly comparable results to QAT. This highlights the effectiveness of our approach in addressing the quantization challenges of Yolov6 and maintaining high accuracy. Overall, our methods demonstrate their ability to improve the quantization performance of Yolov6 models, achieving competitive accuracy compared to QAT and outperforming AdaQuant.
\begin{table}[!th]
  \caption{Detection results of YoLov6 Series on COCO. "v1" is the quantization unfriendly models.}
  \begin{center}
      \begin{tabular}{lccccr}
      \hline
      \makecell[c]{Model} &\makecell[c]{mAp(\%) \\ FP32} & \makecell[c]{mAP(\%)\\Naive} &\makecell[c]{mAP(\%) \\ Adaq} & \makecell[c]{mAP(\%) \\ Ours} \\
      \hline
      \makecell[c]{Yolov6n} & 36.2 & 34.1 & 35.4 & \textbf{35.6} \\
      \makecell[c]{Yolov6t} & 40.9 & 39.8 & 40.5 & 40.5 \\ 
      \makecell[c]{Yolov6s} & 43.9 & 42.5 & 43.2 & \textbf{43.3} \\
      \hline
      \makecell[c]{Yolov6nv1}& 35.0 & 33.0 & 34.4 & 34.4 \\
      \makecell[c]{Yolov6tv1}& 41.3 & 34.0 & 38.1 & \textbf{40.9} \\
      \makecell[c]{Yolov6sv1}& 43.1 & 35.2 & 41.9 & \textbf{42.3} \\
      \hline
      \end{tabular}
      \end{center}
      \label{Detect_exp}
\end{table}
\subsection{Ablation study}
\subsubsection{The effectiveness of our methods} 
The experiments conducted on RepVGGA0 and RepVGGB0 quantized to 6-bit demonstrate the effectiveness of the proposed methods in improving the accuracy of quantized reparameterized models. The results are summarized in Table \ref{ME_exp}.

Vanilla AdaQuant faces challenges in optimizing models with outliers when using Mean Squared Error (MSE) as the measurement. This leads to greater accuracy degradation in RepVGGB0 compared to RepVGGA0, likely due to the presence of larger outliers in RepVGGB0.

Our proposed method achieves similar accuracy loss for both RepVGGA0 and RepVGGB0, indicating that the accuracy degradation is primarily influenced by the presence of outliers. Given that RepVGGB0 has more parameters, the larger outliers in RepVGGB0 likely contribute to its greater accuracy degradation compared to RepVGGA0.
In conclusion, the experimental results demonstrate that each method can improve the accuracy of quantized reparameterized models. The specific impact of each method may vary depending on the presence of outliers.
\begin{table}[!th]
  \caption{Ablation study results of our method.}
  \begin{center}
  \begin{tabular}{lccr}
  \hline
  \makecell[c]{Method} & \makecell[c]{RepVGGA0 \\ Accuracy} & \makecell[c]{RepVGGB0 \\ Accuracy} \\
  \hline
  \makecell[c]{min/max} & 2.01 & 2.21 \\ 
  \makecell[c]{AdaQuant+MAE} & 65.15 & 69.06 \\
  \makecell[c]{+QPRep} & 69.06(\textbf{+3.91}) & 72.71(\textbf{+3.65}) \\
  \makecell[c]{+ABC} & 69.85(\textbf{+0.79}) & 73.06(\textbf{+0.35}) \\
  \hline
  \makecell[c]{Full precision} & 72.41 & 75.15 \\
  \hline
  \end{tabular}
  \end{center}
  \label{ME_exp}
  \vskip -0.1in
\end{table}

\subsubsection{Analysis on MAE and MSE.} We choose the RepVGGB0 model quantized with OMSE \cite{choukroun:low} that uses MAE and MSE to compare the effect of MAE and MSE. We plot the information entropy of quantized value for models optimized with MAE and MSE at Fig. \ref{fig_B0_Ent}. For MAE, the information entropy of quantized outputs is larger than MSE. To split the influence of quantization clipping value for the latter blocks, we plot the distribution of second blocks, and we can see that the quantization scale obtained by MAE is smaller than MSE and has bigger information entropy. To further look into the influence of clipping value, we plot the distribution of the nineth block, we can see that the quantized values cluster below 50 for MSE, but for MAE, the value is more uniform.

In the experiment, the RepVGGB0 model was quantized using the OMSE method, which utilizes both Mean Absolute Error (MAE) and Mean Squared Error (MSE) for comparison. The information entropy of the quantized values was plotted for models optimized with MAE and MSE, as shown in Fig. \ref{fig_B0_Ent}. It can be observed that the information entropy of the quantized outputs is larger when using MAE compared to MSE. This indicates that the quantization process based on MAE introduces more randomness and diversity in the quantized values. To ignore the influence of quantization clipping values on subsequent blocks, the distribution of the second block's values was plotted. It was observed that the quantization scale obtained by MAE is smaller than that obtained by MSE, resulting in a larger information entropy. This suggests that MAE-based quantization allows for a more suitable quantization range, leading to more diverse representations. Furthermore, the distribution of the ninth block's values was plotted to examine the influence of clipping values. For models optimized with MSE, the quantized values tend to cluster below 50. In contrast, for models optimized with MAE, the distribution of values is more uniform, indicating a more balanced spread of quantized representations. These observations highlight the differences between MAE and MSE-based optimization in the quantization process. MAE-based quantization introduces more diversity in quantized values and allows for a wider range of quantization representations, potentially leading to improved model performance in certain scenarios.

\begin{figure}[!tbh]
  \centering
  \centerline{
      \includegraphics[width=1\columnwidth]{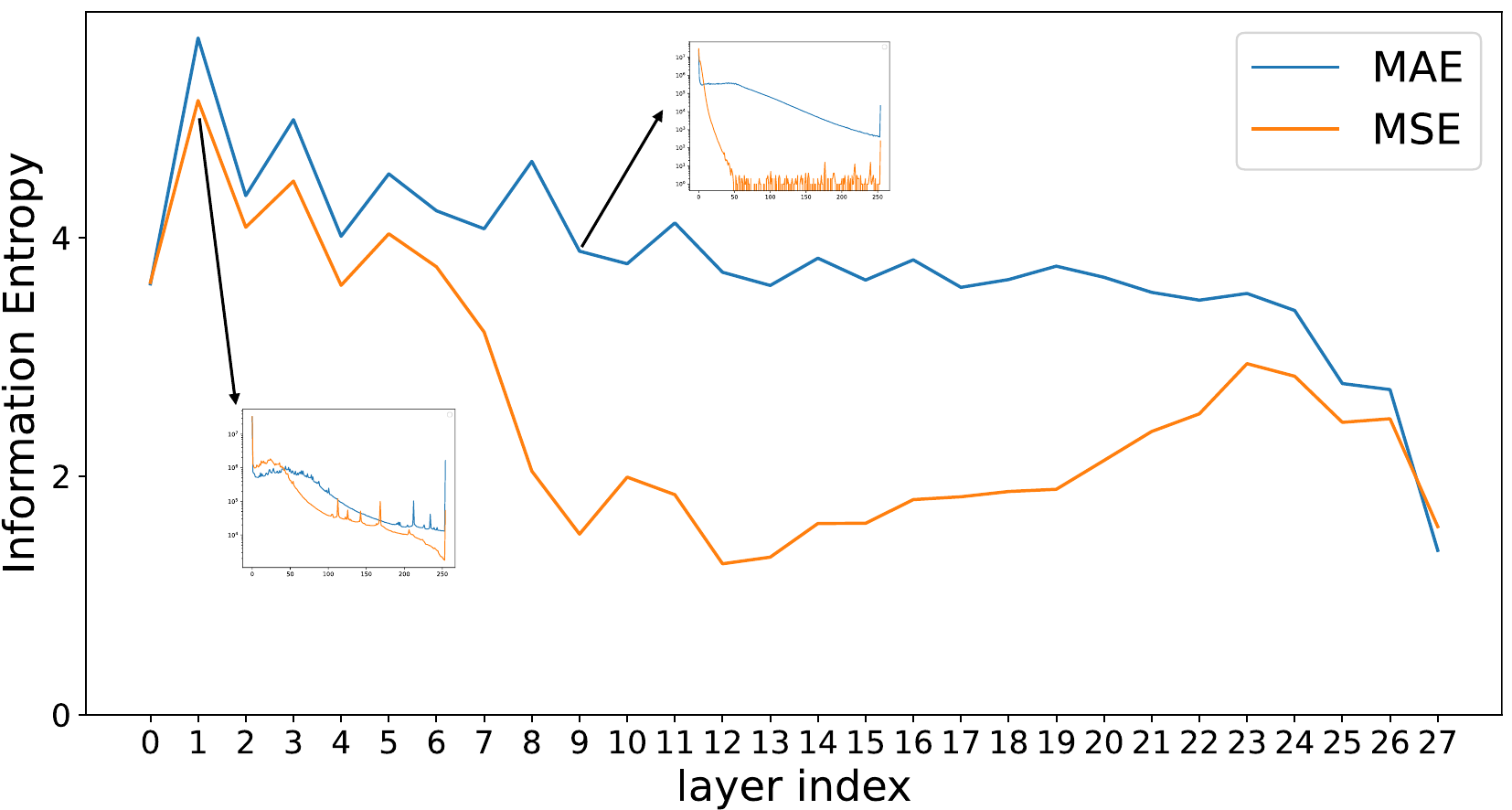}
  }
  \centering
  \caption{Information Entropy of RepVGGB0 quantized with MAE and MSE.}
  \label{fig_B0_Ent}
\end{figure}

\subsubsection{Analysis on QPRep} 
In the experiment, two models of RepVGGB0 were trained, one with QPRep and the other without QPRep. The distributions of the output from two blocks were examined to understand the effect of QPRep on outliers. From Fig. \ref{fig_DARep}, it can be concluded that QPRep helps amplify the features and capture the outliers. The presence of outliers is more prominent in the model trained with QPRep, indicating that QPRep enhances the representation of outliers. Analyzing alongside Fig. \ref{fig_range} and Fig. \ref{fig_DARep}, the ratio of the scale $\eta$ learned with QPRep corresponds to the ratio of outliers before and after quantization clipping. For the 9th block, where outliers are introduced, the quantization clipping does not significantly influence the output. Therefore, QPRep has little effect on this block. However, for the 23rd block, it can be observed that the quantization clipping almost estimates the outliers and the output with QPRep is amplified compared to the output without QPRep. This demonstrates that QPRep can amplify the outputs and maintain the effect of outliers. These findings indicate that QPRep plays a role in amplifying the representation of outliers and maintaining their impact on the model's output.

\begin{figure}[!tbh]
  \centering
  \subfloat[]{\includegraphics[width=0.45\columnwidth, trim=20 20 20 80]{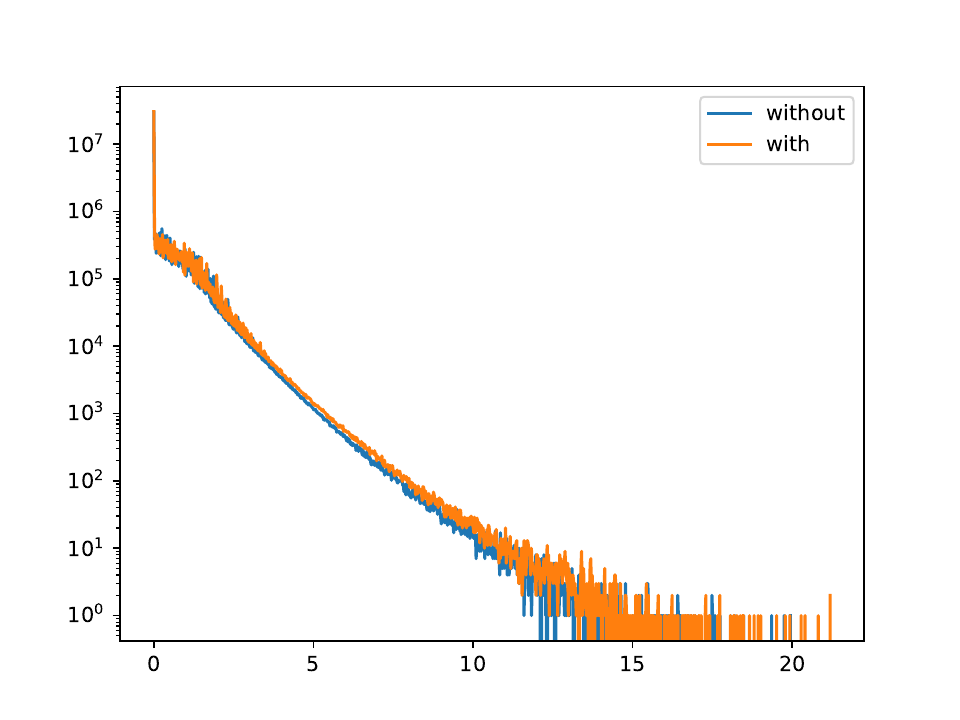}}
  \subfloat[]{\includegraphics[width=0.45\columnwidth, trim=20 20 20 80]{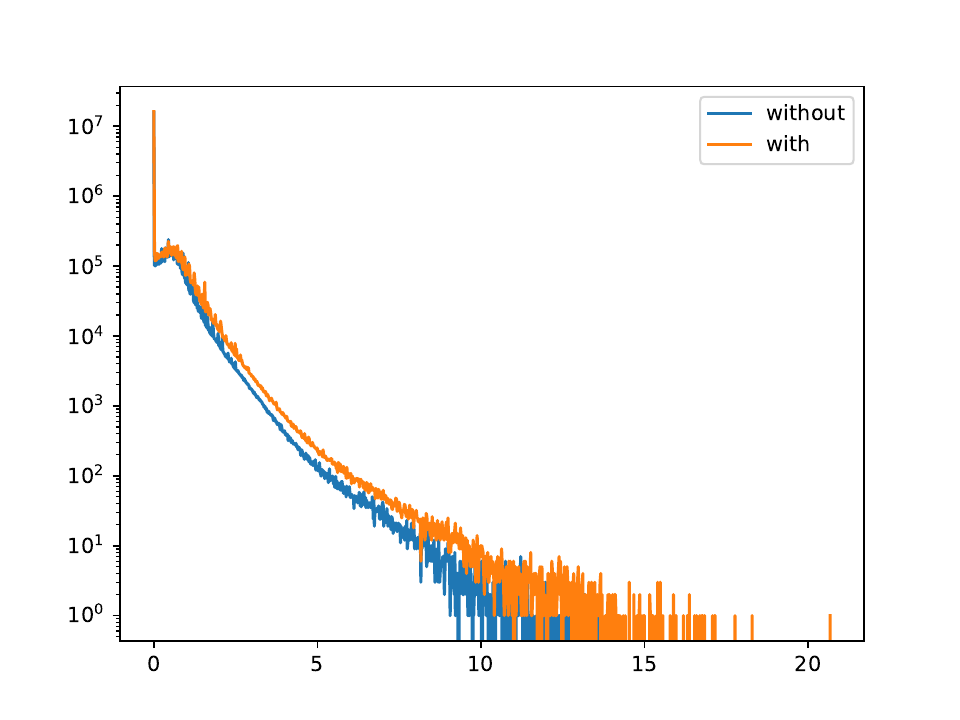}}
  \centering
  \caption{(a) is the output distribution of the 9th block, 
  (b) is the output distribution of the 23th block.}
  \label{fig_DARep}    
  \vskip -0.1in
\end{figure}

\subsubsection{Analysis on ABC} 
In this experiment, we changed the combination of measurements to compare the effectiveness of different combinations. From Table \ref{ABC_exp_2}, it can be observed that for most RepVGG models, MAE outperforms MSE. 
However, for RepVGGA0, which has minimal outliers, MSE may be more suitable than MAE. Following our paradigm, we changed both measurements to MSE. We found that incorporating stage knowledge is beneficial for most networks. 
According to Proposition 1 and Fig. \ref{fig_range}, the outliers of stage outputs are usually smaller than the outliers of layers output. 
Using MAE as the measurement is usually less affected by outliers. These results suggest that considering the MAE of the output, along with the stage output, can be effective in quantization, as it reduces the impact of outliers and improves the accuracy of quantized models.

\begin{table}[h]
  \caption{Results of different measuremant for ABC. Front part means the measuremant of layer outputs when later means the measuremant of stage outputs. }
  \begin{center}
      \begin{tabular}{lcccccr}
      \hline
      \makecell[c]{Model} & MAE+None & MSE+None & MSE+MSE  & MAE+MAE \\
      \hline
      \makecell[c]{RepVGGA0} & 65.15 & 67.05 & 65.90 & 67.83 \\
      \makecell[c]{RepVGGA1} & 69.31 & 67.42 & 70.38 & 70.60 \\ 
      \makecell[c]{RepVGGA2} & 72.43 & 71.51 & 73.03 & 73.55 \\
      \makecell[c]{RepVGGB0} & 69.06 & 68.74 & 70.06 & 70.34 \\
      \makecell[c]{RepVGGB1} & 72.63 & 72.64 & 70.17 & 74.75 \\
      \hline
      \end{tabular}
      \end{center}
      \label{ABC_exp_2}
      \vskip -0.1in
\end{table}

\section{Conclusion}
In this paper, we have analyzed the presence of outliers in reparameterization models and proposed an accurate Post-training Quantization algorithm specifically designed for such models. We have observed that the outliers in reparameterization models are sample-specific and channel-specific, and their size does not have a strong correlation with the accuracy. To address the channel-specific outlier issue, we have introduced an affine layer to scale the output of the channel, effectively protecting it during quantization. For the sample-specific outlier problem, we have replaced the traditional MSE measurement with MAE, which provides a more accurate gradient for distillation. Furthermore, we have presented a novel calibration paradigm called Across-block Calibration, which takes into account the information from the following blocks to overcome the limitations of MAE in terms of gradient accuracy. This paradigm allows us to introduce across-block information to improve the quantization process for the current block. By employing our quantization algorithm, we have demonstrated the ability to quantize reparameterized models into INT8, with accuracy comparable to that of full-precision models. Moreover, we have pushed the limits of quantization further by achieving higher accuracy with INT6 precision, approaching the performance of full-precision models.
It's worth noting that our method may not perform well on reparameterized depth-wise convolutions for lower-bit quantization. In addition, our focus is on the VGG-like models, exploring the application of our method in more complex network architectures, and addressing depth-wise convolution problems, which is a promising direction for future research. We believe that even simple network architectures possess huge potential to achieve performance on par with more complex structures.

\section{Acknowledgments}
\noindent This work was supported by National Key Research and Development Program of China (2022YFC3602601), and Key Research and Development Program of Zhejiang Province of China (2021C02037).

{\appendix[Compararison of quantization for post-add batch normalization and pre-add batch normalization.]
\label{CompNorm}
The Post-add BatchNorm technique is effective in normalizing the sum of branches' outputs and can help limit the presence of outliers in the sum. However, it has been observed, as discussed in RepVGG\cite{ding:repvgg}, that applying BatchNorm after the addition operation can have a negative impact on the accuracy of the model. On the other hand, the Pre-add BatchNorm normalizes the output of each branch before the addition operation takes place. This means that the normalization is applied to each branch individually, potentially including the outliers present in each branch's output. It is important to consider the trade-off between reducing outliers and maintaining accuracy when choosing between Post-add and Pre-add BatchNorm. While Post-add BatchNorm can limit the outliers in the sum, it may have a detrimental effect on accuracy. On the other hand, Pre-add BatchNorm may include outliers in the normalization process but may provide better accuracy results. The choice between these techniques depends on the specific model and the balance between reducing outliers and preserving accuracy. It is recommended to experiment with both approaches and evaluate their impact on the particular model and task at hand.
\begin{figure}[!h]
  \includegraphics[width=0.9\columnwidth]{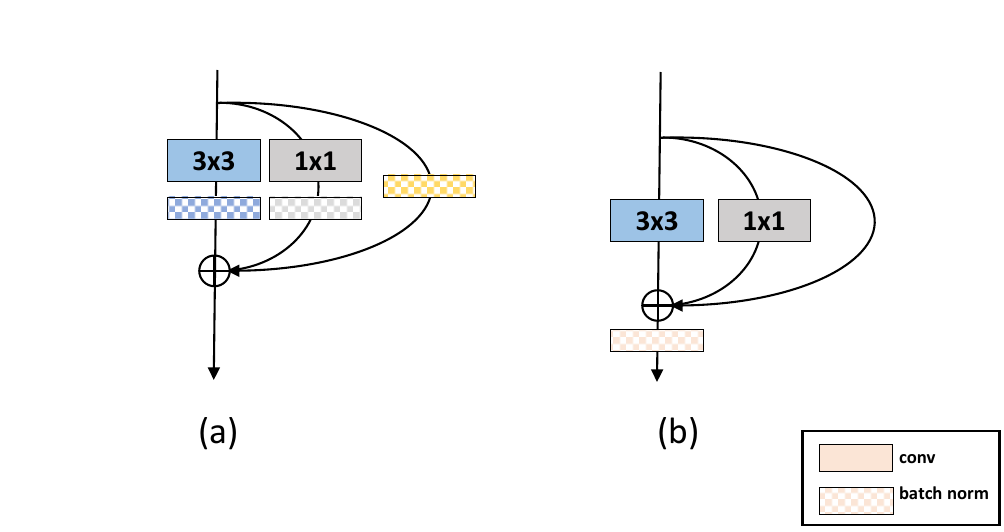}
  \caption{Difference of Post-Add BatchNorm and Pre-Add BatchNorm.(a) is Pre-Add BatchNorms and (b) is Post-add BatchNorm.}
\end{figure}

We trained a RepVGGA0 model with post-add batch normalization and applied the OMSE algorithm to quantize it. The experimental results showed that compared to the pre-add batch normalization approach, the post-add batch normalization technique led to a higher loss in accuracy for the full precision models, with approximately 1\% degradation. Interestingly, when both models were quantized, the pre-add batch normalization model showed no significant loss of accuracy, while the post-add batch normalization model suffered a significant loss of approximately 20\%. These findings indicate that the post-add batch normalization technique is less compatible with the quantization process, resulting in a larger degradation in accuracy compared to the pre-add approach. Therefore, careful consideration should be given to the choice of batch normalization technique when aiming to achieve accurate quantization results for RepVGGA0 models. Further investigation is necessary to understand the underlying reasons behind this discrepancy and explore potential solutions or alternative approaches to improve the quantization performance of post-add batch normalization in RepVGG models.
\begin{table}
  \caption{The 8-bit quantization results of post-add BatchNorm and Pre-add BatchNorm RepVGGA0. }
  \begin{center}
  \begin{tabular}{lccr}
    \hline
    \makecell[c]{Model} & Post-add BatchNorm & Pre-add BatchNorm \\
    \hline
    \makecell[c]{W/O REP} & 71.45 & 71.75  \\
    \makecell[c]{REP} & 71.35 & 50.31 \\ 
    \makecell[c]{FP32} & 71.48 & 72.41 \\
    \hline
    \end{tabular}
    \label{BN_exp}
  \end{center}
    \vskip -0.1in
\end{table}

}

\vfill

\end{document}